\documentclass{article}

\usepackage[utf8]{inputenc} 
\usepackage[T1]{fontenc}    
\usepackage{hyperref}       
\usepackage{url}            
\usepackage{amsfonts}       
\usepackage{nicefrac}       
\usepackage{microtype}      
\usepackage{amssymb,latexsym,amsmath,amsthm}
\usepackage{euscript}
\usepackage{xspace}
\usepackage{graphicx}
\usepackage[usenames,table,dvipsnames]{xcolor}
\usepackage{pgf,tikz}
\usepackage{natbib}
\usepackage[margin=1in]{geometry}
\usepackage{authblk}

\usepackage{paralist} 
\usepackage{enumitem} 

\usepackage{booktabs} 
\usepackage{multicol}
\usepackage{multirow}

\usepackage{wrapfig}

\renewcommand{\c}[1]{\ensuremath{\EuScript{#1}}}
\renewcommand{\b}[1]{\ensuremath{\mathbb{#1}}}
\providecommand{\eps}{{\varepsilon}}%

\newcommand{\bw}{{\ensuremath{\bf w}}}
\newcommand{\bu}{{\ensuremath{\bf u}}}
\newcommand{\bx}{{\ensuremath{\bf x}}}

\newcommand{\norm}[1]{\left\lVert #1 \right\rVert}
\newcommand{\R}{\b{R}}

\newcommand{\p}[1]{\left(#1\right)}
\newcommand{\pb}[1]{\left[#1\right]}

\newcommand{\pA}[1]{\left\langle#1\right\rangle}

\theoremstyle{plain}

\newtheorem{theorem}{Theorem}
\newtheorem{cor}{Corollary}
\newtheorem{lemma}{Lemma}

\newcommand{\cc}[2]{
\pgfmathsetmacro{\PercentColor}{100 * (#1 - #2)/ (100 - #2)}
\xdef\PercentColor1{\PercentColor}
\cellcolor{Green!\PercentColor1!BrickRed} {\scriptsize #1}
}

\newcommand{\papertitle}{Learning In Practice: Reasoning About Quantization}

\newcommand{\decorateText}[3]{} 

\title{\papertitle}

\author{Annie Cherkaev}
\author{Waiming Tai}
\author{Jeff Phillips}
\author{Vivek Srikumar}
\affil{School of Computing, University of Utah}

\date{}

\begin{document}

\maketitle

\begin{abstract}
There is a mismatch between the standard theoretical analyses of statistical machine learning and how learning is used in practice. The foundational assumption supporting the theory is that we can represent features and models using real-valued parameters. In practice, however, we do not use real numbers at any point during training or deployment.  Instead, we rely on discrete and finite quantizations of the reals, typically floating points.  In this paper, we propose a framework for reasoning about learning under arbitrary quantizations.  Using this formalization, we prove the convergence of quantization-aware versions of the Perceptron and Frank-Wolfe algorithms.  Finally, we report the results of an extensive empirical study of the impact of quantization using a broad spectrum of datasets.
\end{abstract}

\section{Introduction}
\label{sec:intro}

Machine learning abounds with theoretical guarantees (e.g., convergence of algorithms) which assume we work with real numbers. However, in practice, every instantiation of the algorithms necessarily uses discrete and finite approximations of real numbers; our hardware is discrete and finite. Such representations are sparse in the space of real numbers. As a consequence, most real numbers are not precisely represented. Does this fact pose problems for learning?

On commodity hardware, learning algorithms typically use 64, 32, or more recently 16 bit floating point numbers.
These approximations are dense enough that, empirically, the guarantees appear to hold. For example, with a $b$ bit representation, $d$-dimensional linear models exist in spaces with $2^{bd}$ distinct points. Typical values of $b$ give sufficiently close approximations to $\R^d$, and learning is reliable, especially after data pre-processing such as normalization.
Floating points are convenient, ubiquitous and portable. Yet, we argue that, machine learning applications present both the need {\em and} the opportunity to rethink how real numbers are represented. With 64 bits and $d$ dimensions, we can distinguish $2^{64d} \p{\approx 10^{19d}}$ points; but learning may not need such high fidelity. The possibility of guaranteed learning with {\em much} coarser numeric representations such as the ones in figure \ref{fig:lattice}, and perhaps even customized ones that are neither fixed nor floating points, could allow for more power efficient customized hardware for learning.

 \begin{figure}[t]
  \begin{center}
    \includegraphics[width=0.5\textwidth]{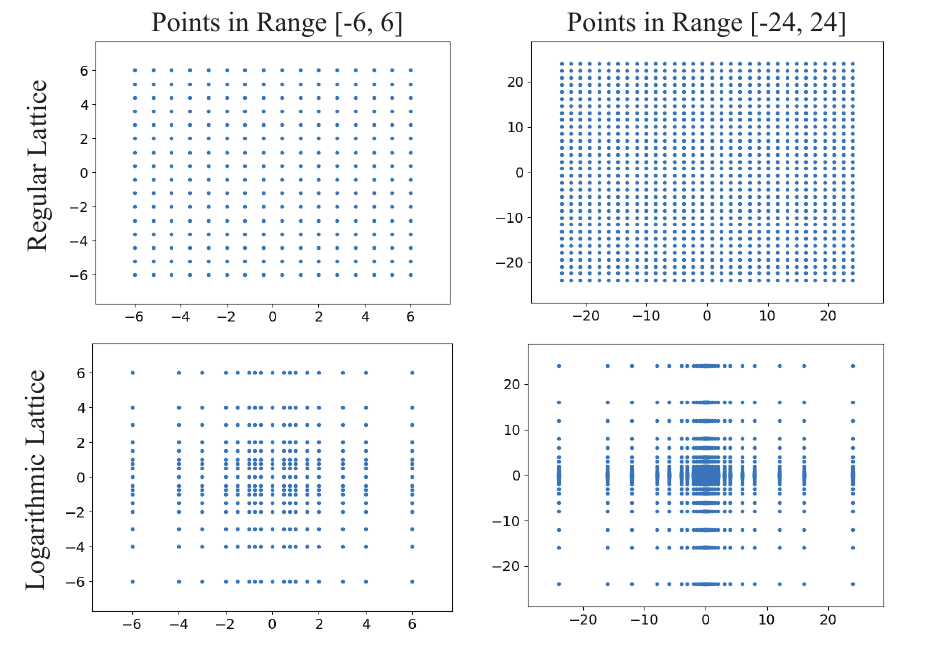}
  \end{center}
\caption{Precisely representable points in the ranges $[-6,6]$ (left) and $[-24, 24]$ (right) under different quantizations. In the top images, these points are spaced in a regular lattice, akin to fixed points. The bottom images show a logarithmically spaced set of points, similar to floating points. Other than these two well studied numeric representations, we could also choose a quantization that is customized to the task, domain or hardware at hand.
}
\label{fig:lattice}
\end{figure}

Moving away from general purpose numeric representations can vastly impact various domains where learning is making inroads. For example, embedded systems and personal devices are resource and power limited.
Datacenters are not resource limited, but the sheer scale of data they operate upon demands sensitivity to the cost of powering them. Both applications 
can reap power saving benefits from custom hardware with efficient custom numeric representations. However, using ad-hoc representations risks unsound learning: specializing hardware for learning requires guarantees. 

Are the standard fine-grained representations needed for guaranteed learning? Can we learn with coarser quantization of the feature and parameter space? For example, figure \ref{fig:lattice} shows examples of different quantizations of a two-dimensional space. While we seek to represent the (bounded) infinite set on the plane, only the points shown in the figure actually exist in our representation. Each representation is effectively blind to the spaces between the representable points --- both features and parameters are constrained to the quantized set. What does it mean to learn over this quantized set?

We present a framework for reasoning about learning under any arbitrary quantization that consists of {\em atoms} -- the finite subset of $\R^d$ which can be represented precisely. Our framework includes not only floating and fixed point representations of numbers, but custom numeric representations as well. Both examples and learned parameters are rounded to the nearest atom. 
We formalize the operators needed over the atoms to define several learning algorithms.

We study two broad families of learning algorithms under our framework. We present a quantization-aware Perceptron mistake bound that shows that, the Perceptron algorithm, despite quantization, will converge to a separating set of (quantized) parameters. We also show convergence guarantees for the Frank-Wolfe algorithm defined over the quantized representations.
Finally, we present a set of empirical results on several benchmark datasets that investigate how the choice of numeric representation affects learning. Across all datasets, we show it is possible to achieve maximum accuracy with much fewer number of atoms than mandated by standard fixed or floating points. Furthermore, we show that merely adjusting the number of bits that we allow for our representations is not enough. The actual points that are precisely representable---i.e., the choice of the atoms---is equally important, and even with the same number of bits, a poor choice of the atoms can render datasets unlearnable.

In summary, we present:
\begin{enumerate}
\item A general framework for reasoning about learning under quantization,
\item theoretical analysis of a family of algorithms that can be realized under our framework, and,
\item experiments with Perceptron on several datasets that highlights the various effects of quantization.
\end{enumerate}

\section{A Framework for Formalizing Quantization}
\label{sec:framework}


In this section, we will define a formalization of quantized representations of real vectors which not only includes floating and fixed points, but is also flexible enough to represent custom quantization. To do so, let us examine the operations needed to train linear classifiers, where the goal is to learn a set of parameters $\bw \in \R^d$. As a prototypical member of this family of learning algorithms, consider the Perceptron algorithm, the heart of which is the Perceptron update: For an example represented by its features $\bx \in \R^d$ with a label $y \in \{-1, +1\}$, we check if the inner product $\langle \bw, \bx\rangle$ has the same sign as $y$. If not, we update the weights to $\bw + y\bx$.
%
The fundamental currency of such an algorithm is the set of $d$-dimensional vectors which represent feature vectors $\bx$ and the learned classifier $\bw$.



We define a quantization via a finite set $A$ of $m$ {\em atoms} $\{a_1, a_2, \cdots, a_m\}$. Each atom precisely represents a unique point in $\R^d$ --- e.g., the points in the examples in figure \ref{fig:lattice}.  Conceptually, we can now think of the learning algorithm as operating on the set $A$ rather than $\R^d$.
Instead of reasoning about learning over vectors, we need to reason about learning over this finite set of representable values. Our abstraction frees us from implicit geometric assumptions we may make when we think about vectors, such as requiring that each dimension contain the same number of representable points. This allows us to model not only the familiar fixed and floating point representations, but also task-dependent custom numeric representations which contain irregularly spaced points.
%

Given a set of atoms, we now need to define the operators needed to define learning algorithms. To support algorithms like the Perceptron algorithm,  we need three operations over set of atoms --- we need to be able to
\begin{inparaenum}[(a)]
\item compute sign of the dot product of two atoms,
\item add two atoms to produce a new atom, and
\item multiply an atom by a real number. 
\end{inparaenum}
Note that, despite atoms being associated with real vectors, we cannot simply add or scale atoms because the result may not be representable as an atom.

To provide a formal basis for these operators, we will define two functions that connect the atoms to the real vector space. For any atom $a \in A$, we will refer to the associated point in $\R^d$ as its {\em restoration}. The restoration $r : A \to \R^d$ is maps atoms to their associated real valued points.
For brevity, if it is clear from the context, we will simplify notation by treating atoms $a$ as vectors via an implicit use of the restoration function.
For any point that is not precisely representable by a set $A$, we need to be able to map it to one of the atoms. We will refer to the function $q: \R^d \to A$ that maps any point in the vector space to an atom as the {\em quantization} of the point.

Thus, we can define a quantization of $\R^d$ via the triple $(A, q, r)$ comprising of the set of atoms $A$, a quantization function $q$ and a restoration function $r$.
The functions $q$ and $r$ give us natural definitions of the {\em intended} semantics of the operations described above and will drive the analysis in \S\ref{sec:bounds}. Note that while these functions formally define a quantization, its implementation cannot explicitly use them because the space $\R^d$ is not available.
Our formalization includes regular lattices such as fixed point, logarithmic lattices such as floating point, as well as more general lattices. For instance, the points in the regular or logarithmic lattices of figure \ref{fig:lattice} can be taken as the set of atoms $A$.

Most of $\R^d$ --- which is infinite --- is simply too far from any atom to be useful, or even encountered during training.  So, we restrict our discussion to a continuous subset $M \subset \R^d$ that contains points of interest; for instance, $M$ could be a ball of sufficiently large radius centered at the origin.  We will assume that all atoms are in $M$.

Since atoms are precisely representable, restoring them via $r$ induces no error. That is, for any $a \in A$, we have $q(r(a)) = a$. The reverse is not true; restoring the quantization of a point $\bx \in M$ via $r(q(\bx))$ need not preserve $\bx$. Intuitively, the gap between $r(q(\bx))$ and $\bx$ should not be arbitrarily large for us to maintain fidelity to the reals. To bound this error, we define the {\em error} parameter $\delta$ of a quantization as
\begin{equation}
  \label{eq:1}
  \delta = \max_{\bx \in M} \|\bx - r(q(\bx))\|.
\end{equation}
Defining the quantization error by a single parameter $\delta$ is admittedly crude; it does not exploit potential variable density of atoms (e.g., with logarithmic lattices in figure \ref{fig:lattice}).  However, it allows for a separation from the geometry of the quantization.  For every atom $a \in A$, we could associate a {\em quantization region} $Q_a \subset M$ such that all points in $Q_a$ are quantized to $a$. That is, $Q_a = \{\bx \in M \mid q(\bx) = a\}$. The definition of $\delta$ bounds the diameter of the quantization regions. In the simplest setting, we can assume $Q_a$ is defined as the Voronoi cell of $a$, a convex subset of $M$ that is closer to $a$ than any other atom.  


\section{Quantization-Aware Learning}
\label{sec:bounds}

In this section, we will look at theoretical analyses of various aspects of quantization-aware learning. First, we will show that under quantization, it may be possible for the error induced by a separating hyperplane to be affected by nearly all the quantization regions.
We will then show that class sample complexity bounds hold, and then analyze the Perceptron algorithm. Finally, we will show an analysis of the Frank-Wolfe algorithm. In both algorithms, our results show that, for learning to succeed, the margin $\gamma$ of a dataset should be sufficiently larger than the quantization error $\delta$.

\subsection{Hyperplanes and Quantization}
\label{sec:lemma1}
In general, collected data may be at a finer resolution than the set of atoms, but we argue that ultimately it is natural to study the computation of classifiers (here, linear separators) on data $X$ that is a subset of $A$.  To do so, we can analyze how unquantized separators interact with the atoms. Intuitively, only quantization regions through which the separator passes can contribute to the error. How many such regions can exist?
While separators for Voronoi cells are studied, but finding separators among Voronoi cells that do not intersect too many cells is known to be difficult~\citep{bhattiprolu2016separating}. For large $d$, almost every atom will be affected by any separator.
%
We formalize this for a specific, illustrative setting.

\begin{lemma}
Consider a domain $M \subset \R^d$ that is a cube centered at the origin for some constant $d$. Suppose the set of atoms $A$ correspond to an axis-aligned orthogonal grid of size $m$.   For any atom $a \in A$, let the quantization region $Q_a$ be its Voronoi region.  Then, any linear separator that passes through the origin will be incident to $\Omega(m^{1-1/d})$ quantization regions.
\label{lemma:linear-separator-incidence}
\end{lemma}
\begin{proof}
  Without loss of generality, let the side length of the cube $M$ be $1$, so the side length of each quantization region $Q_a$ is $1/m^{1/d}$. Then, the diameter of each $Q_a$ is $\sqrt{d}/m^{1/d}$.

Now, consider a linear separator $F$ which is a $(d-1)$-dimensional subspace.
Use any orthogonal basis spanning $F$ to define a $(d-1)$-dimensional grid within $F$. Place a set of $\Omega((m^{1/d})^{d-1}) = \Omega(m^{1-1/d})$ grid points on $F$ so that, along each axis of the basis, they are a distance of $\sqrt{d}/m^{1/d}$ apart. 
No two of these grid points can be in the same $Q_a$ because their separation is at least the diameter of a cell.  Thus, at least $\Omega(m^{1-1/d})$ quantization cells of $A$ must intersect the linear separator.
\end{proof}


The bounded diameter of the quantization regions, which plays an important role in this proof, is related to the worst-case error parameter $\delta$.  If these regions are not Voronoi cells, but still have a bounded diameter, the same proof would work. Some quantizations may allow a specific linear separator $F$ not to intersect so many regions, but then other linear separators will intersect $\Omega(m^{1-1/d})$ regions.

\subsection{Sample Complexity}
\label{sec:sample-complexity}

Suppose we have a dataset consisting of training examples (the set $X$) and their associated binary labels $y$. We will denote labeled datasets by pairs of the form $(X,y)$.
The set of training examples $X$ is most likely much smaller than $A$. Lemma \ref{lemma:linear-separator-incidence} motivates 
that we assume that the training examples are precisely representable under the quantization at hand, or have already been quantized; that is, $X \subset A$. This allows us to focus the analysis on the impact of quantization during learning separately. From a practical point of view, this assumption can be justified because we can {\em only} store quantized versions of input features. For example, if features are obtained from sensor readings, notwithstanding sensor precision, only their quantized versions can exist in memory.\footnote{In \S\ref{sec:experiments}, we will show experiments that violate this assumption and achieve good accuracies. Analyzing such situations is an open question.}

Consider a function class (or range space) $(A,\c{F})$ where $\c{F}$ is a function class defining separators over the set of atoms.  For instance, $\c{F}$ could define quantized versions of linear separators or polynomial separators of a bounded degree.
For any set of functions $\c{F}$ defined over the set of atoms, we can define its real extension $\c{F}'$ as the set of functions that agree with functions in $\c{F}$ for all the atoms. That is, $\c{F}' = \{r(F) \mid F \in \c{F}\}$.
Let the VC-dimension of the real extension of the function space $(\R^d, \c{F}')$ be $\nu$.

\begin{lemma}
Consider a labeled set with examples $X \subset A \subset \R^d$ with corresponding labels $y$. Let $\c{F}$ be  a function class such that its real extension  $(\R^d, \c{F}')$ has VC-dimension $\nu$.
Let $(X_1, y_1) \subset (X, y)$ be a random sample from the example set of size $O(\frac{\nu}{\eps} \log \frac{1}{\eps \delta})$ and $(X_2, y_2) \subset (X, y)$ a random sample of size $O(\frac{1}{\eps^2}(\nu + \log\frac{1}{\delta}))$.
Then, with probability at least $1-\delta$,
\begin{enumerate}[nosep]
\item a perfect separator $F_1 \in \c{F}$ on $(X_1,y_1)$ misclassifies at most $\eps$ fraction on $(X, y)$, and,
\item a separator $F_2 \in \c{F}$ on $(X_2,y_2)$ that misclassifies an $\eta$ fraction of points in $(X_2,y_2)$ misclassifies at most an $\eta+\eps$ fraction of points in $(X,y)$.
\end{enumerate}

\end{lemma}
\begin{proof}
Any $F \in \c{F}$ maps without error into $F' \in \c{F}$ --- i.e., $q(r(F))$ separates the same points as $F$. Then, if $x \in A$ is classified correctly by $F$, then $r(x)$ is classified correctly by $r(F) \in \c{F}'$.  Then, since each $x \in X$ maps to a point $r(x) \in \R^d$, the VC-dimension of $(A, \c{F})$ is also at most $\nu$.  Then, the standard sample complexity bounds~\citep{HW87,VC71,LLS01} apply directly to the quantized versions as claimed.
\end{proof}

\subsection{Quantized Perceptron Bound}
\label{sec:quantized-perceptron-bound}

We next consider the classic Perceptron algorithm on a labeled set $X \subset A$ with each $\bx_i \in X$ labeled with $y_i \in \{-1,+1\}$.   For a linear classifier defined by the normal direction $\bw_t \in A$, a mistake is identified as $y_i \langle r(\bw_t), r(\bx_i) \rangle < 0$, which leads to the  update $$\bw_{t+1} = q\p{r(\bw_t) + y_i r(\bx_i)}.$$

Note that quantization error is only incurred on the last step $\bw_{t+1} = q(r(\bw_t) + y_i r(\bx_i))$ and that
\[
\|r(\bw_{t+1}) - (r(\bw_t) + y_i r(\bx_i))\| \leq \delta.
\]
That is, the new normal direction suffers at most $\delta$ error on a quantized update.

We can adapt the classic margin bound, where we assume that the  data is linearly separable with a margin $\gamma$.

\begin{theorem}[Quantized Perceptron Mistake Bound]
Consider a dataset with examples $X \subset A \subset M \subset \R^d$ where $(X,y)$ has a margin $\gamma$.  Suppose we have a representation scheme whose quantization error is $\delta < \gamma$. Assume every example $\bx \in X$ satisfies $\|r(\bx)\| \leq 1$ and $M$ contains a ball of radius $\sqrt{T} = 1/(\gamma - \delta)$, then after $T$ steps the quantized Perceptron will return $\bw_T$ which perfectly separates the data $(X,y)$.
\end{theorem}

\begin{proof}
First, we argue $\|r(\bw_t)\|^2 \leq t$ since at each step $\|r(\bw_t)\|^2$ increases by at most $1$.  That is, in step $t$ with misclassified $(\bx_i,y_i)$ we have
\begin{align*}
  \begin{split}
\|r(\bw_{t+1})\|^2
  & = \langle r(\bw_t) + y_i r(\bx_i), r(\bw_t) + y_i r(\bx_i) \rangle                                                          \\
  & = \langle r(\bw_t), r(\bw_t)\rangle + (y_i)^2 \langle r(\bx_i), r(\bx_i) \rangle \\
  &\quad  + 2 y_i \langle r(\bw_t), r(\bx_i) \rangle \\
  & \leq \|r(\bw_t)\|^2 + 1 + 0.
\end{split}
\end{align*}
Second, we argue that with respect to the max-margin classifier $\bw^* \in \R^d$, with $\|\bw^*\| = 1$, we have $\langle \bw^*, r(\bw_t) \rangle \geq t (\gamma-\delta)$.  On step $t$ with misclassified $(\bx_i,y_i)$, it increases by at least $\gamma-\delta$:
\begin{align*}
  \langle r(\bw_{t+1}), \bw^* \rangle
  & = \langle r(q(r(\bw_t) + y_i r(\bx_i))), \bw^* \rangle                                                                      \\
  & \geq \langle r(\bw_t) + y_i r(\bx_i), \bw^* \rangle - \delta \|\bw^*\|                                                      \\
  & = \langle r(\bw_t), \bw^* \rangle + y_i \langle r(\bx_i), \bw^* \rangle - \delta                                            \\
  & \geq \langle r(\bw_t), \bw^* \rangle + \gamma - \delta.
\end{align*}
Combining these together
$
t (\gamma - \delta) \leq \langle w^*, r(w_t) \rangle \leq \|r(w_t)\| \leq \sqrt{t},
$
and hence $t \leq 1/(\gamma - \delta)^2$, as desired.  If $t$ is larger, then the second claim is violated, and hence there cannot be another mis-classified point.

Also, note that for this to work, $\bw_t$ must stay within $M$.  Since after $t$ steps $\|\bw_t\| \leq \sqrt{t}$, then over the course of the algorithm $\bw_t$ is never outside of the ball of radius $\sqrt{T} = 1/(\gamma - \delta)$, as provided.
\end{proof}

The theorem points out that if the margin of the data is larger than the quantization error, then the mistake bound is $O(\frac{1}{(\gamma-\delta)^2})$. In other words, with a coarse quantization, we may have to pay the penalty in the form of more mistakes. Note that the above theorem does not make any assumptions about the quantization, such as the distribution of the atoms. If such assumptions are allowed, we may be able to make stronger claims, as the following theorem shows.

\begin{theorem}
When $A$ forms a lattice restricted to $M$ (e.g., the integer grid), the origin is in $A$, the data set $X \subset A$, and $M$ contains a ball of radius $\sqrt{T}$, then after $T$ steps, the infinite precision Perceptron's output $w^\infty_T$ and the quantized Perceptron's output $\bw_T$ match in that $r(\bw_T) = \bw^\infty_T$.
\end{theorem}
\begin{proof}
Since $X \subset A$, the only quantization step in the algorithm is on
\[
\bw_{t+1} = q(r(\bw_t) + y_i r(\bx_i)),
\]
for mistake $(\bx_i,y_i)$ with $\bx_i \in X \subset A$.  However, since $A$ is a lattice in $\R^d$, then $r(\bw_t)$ is on the lattice, and $r(\bx_i)$ is on the lattice, and hence, by definition, their sum (or difference if $y_i = -1$) is also on the lattice.  Hence letting $\bw_{t+1} = r(\bw_t) + y_i r(\bx_i)$, we have that $\bw_{t+1} = r(q(\bw_{t+1}))$.

By the condition in the theorem, the set $M$ contains a ball of radius $\sqrt{T}$.  Then, $\bw_t$ never leaves $M$ and never leaves the lattice defining $A$.
\end{proof}

In this case, the mistake bound is $O(1/\gamma^2)$, as in standard the full-precision Perceptron.

\subsection{Frank-Wolfe Algorithm with Quantization}
\label{sec:frank-wolfe}


Next, we will analyze the Frank-Wolfe algorithm on a dataset $(X, y)$.
Initially, when t = 0, we take $\bw_0 = q(r(\bx_i))$ where $\bx_i$ has its label $y_i = +1$ and $\norm{r(\bx_i)}$ is minimal.
In $t$-th step, identify an example $i$ such that
\[
	i = \arg\min_{i'} y_{i'}\pA{  r(\bw_t), r(\bx_{i'})  }
\]
and update
\[
	\bw_{t+1} = q\p{  r\pb{ q\p{  \alpha y_i r\p{\bx_i} } } +   r\pb{ q\p{ (1-\alpha) r\p{\bw_t}  }} }
\]
where
\[
	\alpha = \arg\min_{\alpha' \in [0,1]} \norm{ \alpha' y_i r(\bx_i)  +    (1-\alpha') r(\bw_t)  }
\]
We will refer to this algorithm as the quantized Frank-Wolfe algorithm. Note that the computation of the $\alpha$ requires line search over real numbers. While this may not be feasible in practice, we can show formal learnability guarantees that can act as a blueprint for further analysis, where the update of $\bw$ would be guided by combinatorial search over the atoms.

\begin{theorem}[Quantized Frank-Wolfe Convergence]
  Consider a data set $X \subset A \subset M \subset \R^d$, with quantization error of $\delta$, and where $(X,y)$ has a margin $\gamma$.  Assume every example $\bx \in X$ satisfies $\|r(\bx)\| \leq 1$, then after $T=O(\frac{1}{\sqrt{\gamma\delta}}\log \frac{1}{\eps}+\frac{1}{\eps\gamma})$ steps, the quantized Frank-Wolfe algorithm will return weights $\bw$ which guarantees
  \[
    \min_j y_j\left\langle r(\bx_j), \frac{r(\bw)}{\norm{r(\bw)}} \right\rangle > \gamma - \sqrt{\frac{24\delta}{\gamma}} - \eps.
  \]\label{th:frank-wolfe-theorem}
\end{theorem}


\begin{proof}

The update step in the algorithm allows us to expand $r(\bw_{t+1})$ as
\begin{align*}
  r(\bw_{t+1})
       & =
    r\left(q\left(  r\left( q\left(  \alpha y_i r\left(\bx_i\right) \right) \right) +   r\left(q\left( (1-\alpha\right) r\left(\bw_t\right)  \right)\right) \right) \\
       & =
    (1-\alpha) r(\bw_t) + \alpha y_i r(\bx_i)+ \bu.
\end{align*}
Here $\bu$ is a vector with $\norm{\bu} \leq 3\delta$.
In other words,
\[
  r(\bw_{t+1}) - \bu = (1-\alpha) r(\bw_t) + \alpha y_i r(\bx_i)
\]
Following the analysis of~\citet{gartner2009coresets}, we have
\begin{align*}
 & \norm{r(\bw_t)} - \norm{r(\bw_{t+1}) - \bu}                                                                   \\
  \geq & \frac{\gamma}{8}\p{\norm{r(\bw_t)} - \min_j y_j\langle r(\bx_j), \frac{r(\bw_t)}{\norm{r(\bw_t)}} \rangle}^2 \\
  \geq & \frac{\gamma}{8}\p{\norm{r(\bw_t)} - \gamma}^2.
\end{align*}

We can rearrange the above inequality and simplify it by defining $f_t =
\norm{r(\bw_t)} - \gamma - \sqrt{\frac{24\delta}{\gamma}}$. We get 
\begin{align*}. 
  f_t - f_{t+1}
  & =
    \norm{r(\bw_t)} - \norm{r(\bw_{t+1})} \\
  & \geq
    \norm{r(\bw_t)} - \norm{r(\bw_{t+1})-u} - 3\delta \\
  & \geq
    \frac{\gamma}{8}(\norm{r(\bw_t)}-\gamma)^2 - 3\delta \\
  & =
    \frac{\gamma}{8}(f_t + \sqrt{\frac{24\delta}{\gamma}})^2 -3\delta \\
  & =
    \frac{\gamma}{8}f_t^2 + \sqrt{\frac{3\delta\gamma}{2}} f_t \\
  & \geq
    \sqrt{\frac{3\delta\gamma}{2}} f_t
\end{align*}
That is, $f_{t+1} \leq (1-\sqrt{\frac{3\delta\gamma}{2}})f_t \leq \exp(-\sqrt{\frac{3\delta\gamma}{2}}) f_t$.

Suppose that the algorithm runs for $T_1$ steps, we have
\[
  \norm{r(\bw_{T_1})} \leq \gamma +\sqrt{\frac{24\delta}{\gamma}} + \exp(-T_1\sqrt{\frac{3\delta\gamma}{2}}).
\]
Furthermore, if ${T_1} = O(\frac{1}{\sqrt{\gamma\delta}}\log \frac{1}{\eps})$, then we can bound the norm of the reconstructed weight vector as $\norm{r(\bw_{T_1})} \leq \gamma + \sqrt{\frac{24\delta}{\gamma}} +\eps$.

In order to show the required bound, we need to consider the case where the algorithm runs for $O(\frac{1}{\eps\gamma})$ more steps. We will set up a contradiction to show this. Recall that
\[
  f_t - f_{t+1} \geq \frac{\gamma}{8}\left(\norm{r(\bw_t)} - \min_j y_j\left\langle r(\bx_j), \frac{r(\bw_t)}{\norm{r(\bw_t)}} \right\rangle\right)^2 - 3\delta.
\]
Suppose $\norm{r(\bw_t)} - \min_j y_i\left\langle r(\bx_j), \frac{r(\bw_t)}{\norm{r(\bw_t)}} \right\rangle > \eps + \sqrt{\frac{24\delta}{\gamma}}$ for $t > {T_1}$.
This allows us to simplify the bound for $f_t - f_{t+1}$ above to $\frac{\gamma}{8}(\eps + \sqrt{\frac{24\delta}{\gamma}})^2 - 3\delta = \eps(\frac{\gamma\eps}{8} + \sqrt{\frac{3\delta\gamma}{2}})$.
If the algorithm runs ${T_2} = \frac{8}{\eps\gamma+\sqrt{96\delta\gamma}}$ more steps,
\[
  f_{{T_1}+{T_2}}-f_{T_1} \geq {T_2}\eps(\frac{\gamma\eps}{8} + \sqrt{\frac{3\delta\gamma}{2}}) = \eps
\]
However, it leads to a contradiction of
\[
  \eps > f_{{T_1}+{T_2}} > f_{{T_1}+{T_2}}-f_{T_1}> \eps
\]
That means that $\norm{r(\bw_t)} - \min_j y_j\langle r(\bx_j), \frac{r(\bw_t)}{\norm{r(\bw_t)}} \rangle < \eps + \sqrt{\frac{24\delta}{\gamma}}$ for some $t \in [{T_1}, {T_1}+{T_2}]$.

Namely, if the algorithm runs $O(\frac{1}{\sqrt{\gamma\delta}}\log
\frac{1}{\eps}+\frac{1}{\eps\gamma+\sqrt{\delta\gamma}}) =
O(\frac{1}{\sqrt{\gamma\delta}}\log
\frac{1}{\eps}+\frac{1}{\eps\gamma})$ steps, it returns a vector $\bw$
that guarantees
\[
  \min_j y_j\left\langle r(\bx_j), \frac{r(\bw)}{\norm{r(\bw)}} \right\rangle > \gamma - \sqrt{\frac{24\delta}{\gamma}} - \eps.
\]
\end{proof}



As in the quantized Perceptron mistake bound, the proof of the above theorem follows the standard strategy for proving the convergence of the Frank-Wolfe algorithm.
The theorem points out that after $T$ steps, the margin of the resulting classifier will not be much smaller than the true margin of the data $\gamma$, and the gap is dependent on the quantization error $\delta$. Two corollaries of this theorem shed further light by providing additive and multiplicative bounds on the resulting margin if the quantization error satisfies certain properties.
\begin{cor}
Consider a data set $X \subset A \subset M \subset \R^d$, with quantization error of $\delta \leq \eps^2\gamma$, and where $(X,y)$ has a margin $\gamma$.  Assume every example $\bx \in X$ satisfies $\|r(\bx)\| \leq 1$, then after $T=O(\frac{1}{\eps\gamma}\log \frac{1}{\eps})$ steps, the quantized Frank-Wolfe algorithm will return weights $\bw$ which guarantees $\min_j y_j\langle r(\bx_j), \frac{r(\bw)}{\norm{r(\bw)}} \rangle > \gamma - \eps$.
\end{cor}

\begin{cor}
Consider a data set $X \subset A \subset M \subset \R^d$, with quantization error of $\delta \leq \eps^2\gamma^3$, and where $(X,y)$ has a margin $\gamma$.  Assume all $\bx \in X$ satisfies $\|r(\bx)\| \leq 1$, then after $T=O(\frac{1}{\eps\gamma^2}\log \frac{1}{\eps\gamma})$ steps, the quantized Frank-Wolfe algorithm will return $\bw$ which guarantees $\min_j y_j\langle r(\bx_j), \frac{r(\bw)}{\norm{r(\bw)}} \rangle > (1- \eps) \gamma$.
\end{cor}

In essence, these results show that if the worst-case quantization error is small compared to the margin, then quantized Frank-Wolfe will converge to a good set of parameters. As in the quantized Perceptron, we do not make any assumptions about the nature of quantization and the distribution of atoms. In other words, the theorem and its corollaries apply not only to fixed and floating point quantizations, but also to custom quantizations of the reals.

\section{Experiments and Results}
\label{sec:experiments}

In this section, we present our empirical findings on how the choice of numeric representation affects performance of classifiers trained using Perceptron. Specifically, we emulate three types of lattices: logarithmic lattices (like floating point), regular lattices (like fixed point), and custom quantizations which are defined solely by the collection of points they represent precisely.
Our results empirically support and offer additional intuition for the theoretical conclusions from \S\ref{sec:bounds} and investigate sources of quantization error.
Additionally, we also investigate the research question: {\em Given a dataset and a budget of $b$ bits, which points of $\R^d$ should we represent to get the best performance.}

\begin{table*}[t]
\centering
\caption{Characteristics of the datasets used in our experiments.}
\footnotesize
\begin{tabular}{llrrrrrr}
\toprule
          & Feature & Num      & Max Feature & Min Feature & Number of        & Majority & Max      \\
Dataset   & Type    & Features & Magnitude   & Magnitude   & Training/Testing & Baseline & Accuracy \\
\midrule
synth01   & Real    & 2        & 6.9         & 0.9         & 160/40           & 50       & 100      \\
synth02   & Real    & 2        & 2.7         & 0.01        & 80/20            & 50       & 100      \\
mushrooms & Bool    & 112      & Bool        & Bool        & 7,000/1,124      & 52       & 100      \\
gisette   & Real    & 5,000    & 3.2         & 0.001       & 6,000/1,000      & 50       & 97       \\
cod-rna   & Real    & 8        & 1868        & 0.08        & 59,535/271,617   & 66       & 88       \\
farm-ad   & Bool    & 54,877   & Bool        & Bool        & 3,100/1043       & 53       & 88       \\
\bottomrule
\end{tabular}
\label{tab:data-table}
\end{table*}


%
%

\subsection{Quantization Implementation Design Decisions}

Before describing the experiments, we will first detail the design decisions that were made in implementing the quantizers used in experiments.

To closely emulate the formalization in \S~\ref{sec:framework}, we define a quantization via two functions: a quantizer function $q$ (which translates any real vector to a representable atom), and a restoration function $r$ (which translates every atom to the real vector which it represents precisely).

We use 64-bit floating points as a surrogate for the reals. In both the logarithmic and regular lattices, the distribution of lattice points used is symmetric in all dimensions. This means that if we model a bit-width $b$ with $2^b$ lattice points, then the $d$-dimensional feature vectors will exist in a lattice with $2^{bd}$ distinct points.

Fixed point requires specifying a range of numbers that can be represented and the available bits define an evenly spaced grid in that range. For floating points, we need to specify the number of bits used for the exponent; apart from one bit reserved for the sign, all remaining bits are used for the mantissa.

We have also implemented a fully custom quantization with no geometric assumptions. Its purpose is to address the question: if we have more information about a dataset, can we learn with substantially fewer bits?

\paragraph{A Logarithmic Lattice: Modeling Floating Point}

We have implemented a logarithmic lattice which is modeled on a simplified floating point representation. The latest IEEE specification (2008) defines the floating point format for only 16, 32, 64, 128 and 256 bit wide representations, therefore we have adapted the format for arbitrary mantissa and exponent widths. The interpretation of the exponent and the mantissa in our implementation is the same as defined in the standard; the following section further explores which points are representable in this lattice. While the official floating point specification also includes denormalized values, we have chosen to not represent them. In practice denormalized values complicate the implementation of the floating point pipeline which is contrary with our goal of designing power-conscious numeric representations. We have also chosen to not represent  $\pm\infty$; instead our implementation overflows to the maximum or minimum representable value. This behavior is reasonable for our domain because the operation which fundamentally drives learning is considering the sign of the dot product, and bounding the maximum possible magnitude does not influence the sign of the dot product.

\paragraph{A Regular Lattice: Modeling Fixed Point}

We have also implemented a regular lattice which is modeled on a fixed point representation. This lattice is parameterized by the range in which values are precisely represented, and by the density of represented points. The range parameter is analogous to the exponent in floating point (both control the range of representable values), and the density parameter is analogous to the mantissa. Similarly to our floating point implementation, our fixed point representation symmetrically represents positive and negative values, and has the same overflow behavior.

\paragraph{A Custom Lattice: Quantizing With A Lookup Table}

In addition to the logarithmic and regular lattices, we have also implemented a fully custom lattice. This lattice is represented as a lookup table that maps precisely representable points to a boolean encoding. For instance, if we wish to use a bit-width of 2, meaning we can precisely represent 4 points, we can create a table with 4 rows, each of which map a vector in $R^d$ to one of the 4 available atoms. The quantization function for this table quantizer is defined by returning the atom which is mapped to the vector found by performing nearest neighbors on the precisely represented vectors. While a hardware implementation is beyond the scope of this paper, lookup tables can be implemented efficiently in hardware.

\begin{figure*}
\includegraphics[width=1\textwidth]{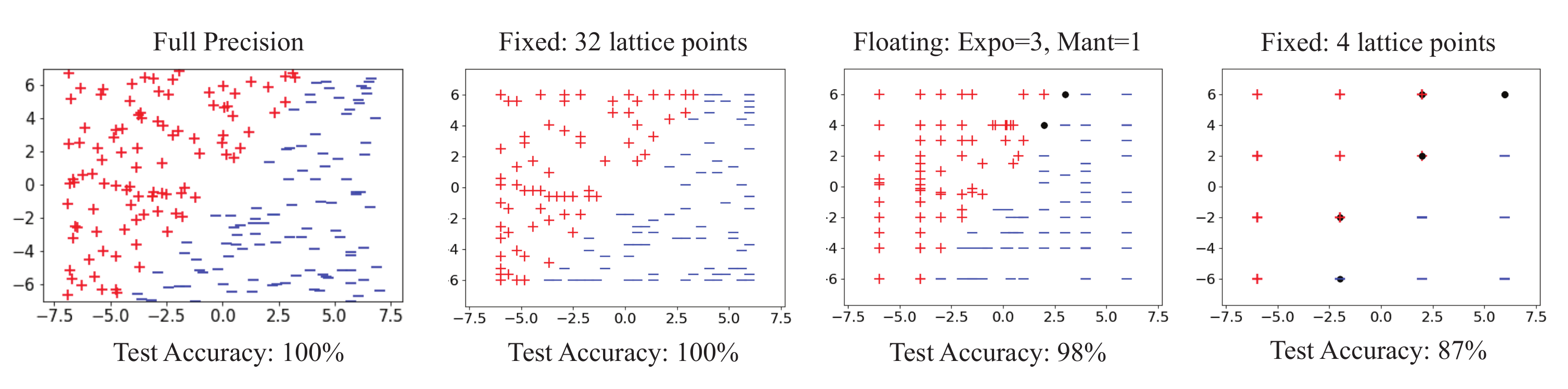}
\caption{The synth01 dataset under various quantizations. Misclassified points are denoted by a black circle. Left to Right: Full Precision; Fixed Point in [-6, 6] with 32 lattice points/dimension; Floating Point with 3 bits of exponent, 1 bit of mantissa; Fixed Point in [-6, 6] with 4 lattice points/dimension.}
\label{fig:lattices}
\end{figure*}

\subsection{Experimental Setup}

To gain a broad understanding of how quantization affects learning, we have selected datasets with a variety of characteristics, such as number of features, number of feature values, and linear separability that may influence learnability. Table \ref{tab:data-table} summarizes them. The majority baseline in table \ref{tab:data-table} gives the accuracy of always predicting the most frequent label, and max accuracy specifies the accuracy achieved by quantizing to 64-bit floats.
\footnote{These datasets are available on the UCI machine learning repository or the {\tt libsvm} data repository.}

We have implemented a generic version of Perceptron that uses a given quantization. For each dataset, we ran Perceptron for 3 epochs with a fixed learning rate of 1 to simplify reasoning about the updates; using a decaying learning rate produces similar results. 

\subsection{Sources of Quantization Error}
\label{sec:is-learning-possible}


\paragraph{Mapping Vectors to Lattice Points}

First, let us use a 2-dimensional synthetic linearly separable dataset to illustrate the effects of quantization. Figure~\ref{fig:lattices} shows the synth01 dataset presented under different quantizations.
The quantization second-to-the-left (Fixed: 32 lattice points) achieves 100\% accuracy while providing only $2^{10}$ possible lattice points as opposed to the $ 2^{128}$ lattice points available under full precision.

What are the sources of error in these quantizations? In the plot second from the right (Floating: expo=3, mant=1) there are two misclassified points that are close to the decision boundary in the full precision plot. These points are misclassified because the quantization has insufficient resolution to represent the true weight vector. In the right-most plot (Fixed: 4 lattice points) some of the misclassified points are plotted as both being a correctly classified and a misclassified point. There are points in the test set with different labels which get quantized to the same point, therefore that lattice point contains both correctly classified test points and misclassified test points. In effect, this quantization is mapping a dataset which is linearly separable in full precision to one which is linearly inseparable under low precision.

\paragraph{Learning, and Not Learning, on Mushrooms}

Table~\ref{tab:mushrooms-results} reports set accuracies on the mushrooms dataset when quantized under a variety of fixed point parameters. The mushrooms dataset is linearly separable, and indeed, we observe that with 256 lattice points per dimension distributed evenly in the interval $[-1, 1]^{112}$ (corresponding to $2^8$ bits for each of the 112 dimensions) we can achieve 100\% test accuracy. Having only 4 or 8 lattice points per dimension, however, is insufficient to find the classifying hyperplane in any of the reported ranges.

Notice that for parameter values in a certain range, the classifier does not learn at all (reporting 50\% accuracy), but in the remaining range, the classifier does fine. This bifurcation is caused by edge-effects of the quantizations; the atoms on the outside corners of the representable points act as sinks. Once the classifier takes on the value of one of these sink atoms, the result of an update with any possible atom snaps back to the same sink atom, so no more learning is possible. The algorithm does not learn under quantizations which had many such sinks; the sinks are an artifact of the distribution of points, the rounding mode and the overflow mode.


\begin{table}
  \centering
  \footnotesize
  \caption{Test set accuracies on the mushrooms dataset for different fixed point quantizations. Green cells are closer to the maximum possible accuracy; redder cells are closer to the majority baseline. }
  \begin{tabular}{r|p{0.4cm}p{0.4cm}p{0.4cm}p{0.4cm}p{0.4cm}p{0.4cm}}
    \toprule
                    & \multicolumn{6}{c}{\footnotesize Number of lattice points per dimension}                                           \\
   Range            & 8                       & 16          & 32          & 64           & 128          & 256              \\
    \midrule
    $[-0.5, 0.5]$   & \cc{52}{48} &\cc{93}{48} & \cc{97}{48} & \cc{98}{48} & \cc{100}{48} & \cc{100}{48} \\
    $[-0.75, 0.75]$ & \cc{48}{48} &\cc{93}{48} & \cc{99}{48} & \cc{100}{48} & \cc{100}{48} & \cc{100}{48} \\
    $[-1,1]$        & \cc{52}{48} &\cc{93}{48} & \cc{95}{48} & \cc{98}{48} & \cc{99}{48} & \cc{100}{48} \\ 
    $[-2,2]$        & \cc{52}{48} &\cc{48}{48} & \cc{88}{48} & \cc{96}{48} & \cc{99}{48} & \cc{100}{48} \\
    $[-4,4]$        & \cc{48}{48} &\cc{48}{48} & \cc{52}{48} & \cc{99}{48} & \cc{100}{48} & \cc{100}{48} \\
    $[-8,8]$        & \cc{48}{48} & \cc{48}{48} & \cc{52}{48} & \cc{48}{48} & \cc{85}{48} & \cc{99}{48} \\
    \bottomrule
  \end{tabular}
  \label{tab:mushrooms-results}
\end{table}


\subsection{Which Atoms Are Necessary?}
\label{sec:experiments:atoms}

Given a dataset, a natural question is: how many bits are necessary to get sufficient classification accuracy? This question is insufficient; in this section, we will discuss why it is not only the number of lattice points, but also their positions, that affect learnability.

\paragraph{The Most Bang For Your Bits}

Table~\ref{tab:gisette-results} reports testing accuracies for different choices of both fixed and floating point parameters that result in the same bit-width for the gisette dataset; table~\ref{tab:farm-results} reports the same for the farm-ad dataset. The table reports wild variation -- from completely unconverged weights reporting 50\% accuracy to well-converged weights reporting 94\% accuracy. With sufficiently many bits (the right-most column) any quantization with sufficiently large range (all rows but the top fixed point row); however it is possible to get high accuracy even at lower bit-widths, if the placement of the atoms is judiciously chosen.

\begin{table}
  \centering
  \footnotesize
  \caption{Test set accuracies on the gisette dataset for different quantizations. For the floating point representations, each row represents a different value of the exponent part. One bit in the bit budget is assigned for the sign and the rest of the bits are used for the mantissa.  Green cells are closer to the maximum possible accuracy, while redder cells are closer to the majority baseline.}   
  \begin{tabular}{r|cccc}
    \toprule
                 & \multicolumn{4}{c}{Bit budget for fixed points}                   \\
    Range        & 6 bits      & 7 bits      & 8 bits      & 9 bits                  \\
    \midrule
    $[-8,8]$     & \cc{90}{50} & \cc{89}{50} & \cc{62}{50} & \cc{50}{50}             \\ 
    $[-16,16]$   & \cc{92}{50} & \cc{93}{50} & \cc{88}{50} & \cc{88}{50}             \\ 
    $[-32,32]$   & \cc{50}{50} & \cc{90}{50} & \cc{94}{50} & \cc{93}{50}             \\ 
    $[-64,64]$   & \cc{50}{50} & \cc{50}{50} & \cc{91}{50} & \cc{94}{50}             \\ 
    $[-128,128]$ & \cc{50}{50} & \cc{50}{50} & \cc{50}{50} & \cc{89}{50}             \\ 
    \bottomrule
  \end{tabular}\hfill
  \begin{tabular}{r|cccc}
    \toprule
    \# exponent  & \multicolumn{4}{c}{Bit budget for floating points} \\
    bits         & 6 bits      & 7 bits      & 8 bits      & 9 bits                  \\
    \midrule
    1            & \cc{57}{50} & \cc{71}{50} & \cc{82}{50} & \cc{82}{50}             \\
    2            & \cc{88}{50} & \cc{81}{50} & \cc{88}{50} & \cc{82}{50}             \\
    3            & \cc{81}{50} & \cc{85}{50} & \cc{83}{50} & \cc{94}{50}             \\
    4            & \cc{50}{50} & \cc{51}{50} & \cc{88}{50} & \cc{94}{50}             \\
    5            & --          & \cc{50}{50} & \cc{50}{50} & \cc{77}{50}             \\
    6            & --          & --          & \cc{50}{50} & \cc{50}{50}             \\
    7            & --          & --          & --          & \cc{50}{50}             \\
    \bottomrule
  \end{tabular}

  \label{tab:gisette-results}
\end{table}


%
 \begin{table}
  \centering
  \footnotesize
    \caption{Test set accuracies on the farm-ads dataset for different quantizations. For the floating point representations, each row represents a different value of the exponent part. One bit in the bit budget is assigned for the sign and the rest of the bits are used for the mantissa.  Green cells are closer to the maximum possible accuracy, while redder cells are closer to the majority baseline.}
  \begin{tabular}{r|cccc}
    \toprule
                 & \multicolumn{4}{c}{Bit budget for fixed points}       \\
    Range        & 9 bits      & 10 bits     & 11 bits     & 12 bits     \\
    \midrule
    $[-8,8]$     & \cc{83}{51} & \cc{82}{51} & \cc{79}{51} & \cc{86}{51} \\ 
    $[-16,16]$   & \cc{77}{51} & \cc{77}{51} & \cc{76}{51} & \cc{87}{51} \\ 
    $[-32,32]$   & \cc{51}{51} & \cc{59}{51} & \cc{66}{51} & \cc{88}{51} \\ 
    $[-64,64]$   & \cc{51}{51} & \cc{51}{51} & \cc{56}{51} & \cc{87}{51} \\ 
    $[-128,128]$ & \cc{51}{51} & \cc{51}{51} & \cc{52}{51} & \cc{88}{51} \\ 
    \bottomrule
  \end{tabular}\hfill
  \begin{tabular}{r|cccc}
    \toprule    
    \# exponent  & \multicolumn{4}{c}{Bit budget for floating points}    \\
    bits         & 7 bits      & 8 bits      & 9 bits      & 10 bits     \\
    \midrule
    1            & \cc{84}{51} & \cc{84}{51} & \cc{84}{51} & \cc{84}{51} \\
    2            & \cc{86}{51} & \cc{89}{51} & \cc{87}{51} & \cc{87}{51} \\
    3            & \cc{87}{51} & \cc{86}{51} & \cc{86}{51} & \cc{86}{51} \\
    4            & \cc{85}{51} & \cc{87}{51} & \cc{89}{51} & \cc{89}{51} \\
    5            & \cc{85}{51} & \cc{85}{51} & \cc{89}{51} & \cc{87}{51} \\
    6            & --          & \cc{85}{51} & \cc{85}{51} & \cc{85}{51} \\
    7            & --          & --          & --          & \cc{85}{51} \\
    \bottomrule
  \end{tabular}
  \label{tab:farm-results}
\end{table}


\paragraph{Normalization \& Quantization}

The cod-rna dataset contains a small number features, but they span in magnitude from 1868 to 0.08. This large range in scale makes cod-rna unlearnable at small bit-widths; it requires both a high lattice density to represent the small magnitude features and sufficient range to differentiate the large magnitude features. We found cod-rna required at least 12 bits under a floating point quantization (1 bit sign, 5 bits exponent, 6 bit mantissa) and at least 11 bits under a fixed point quantization ($2^{11}$ points in $[-2048, 2048]^{11}$). Quantization and normalization are inseparable; the range of feature magnitudes directly influences how large the lattice must be to correctly represent the data.

\paragraph{Low Bitwidth Custom Quantization}

\begin{figure*}[t]
\includegraphics[width=1\textwidth]{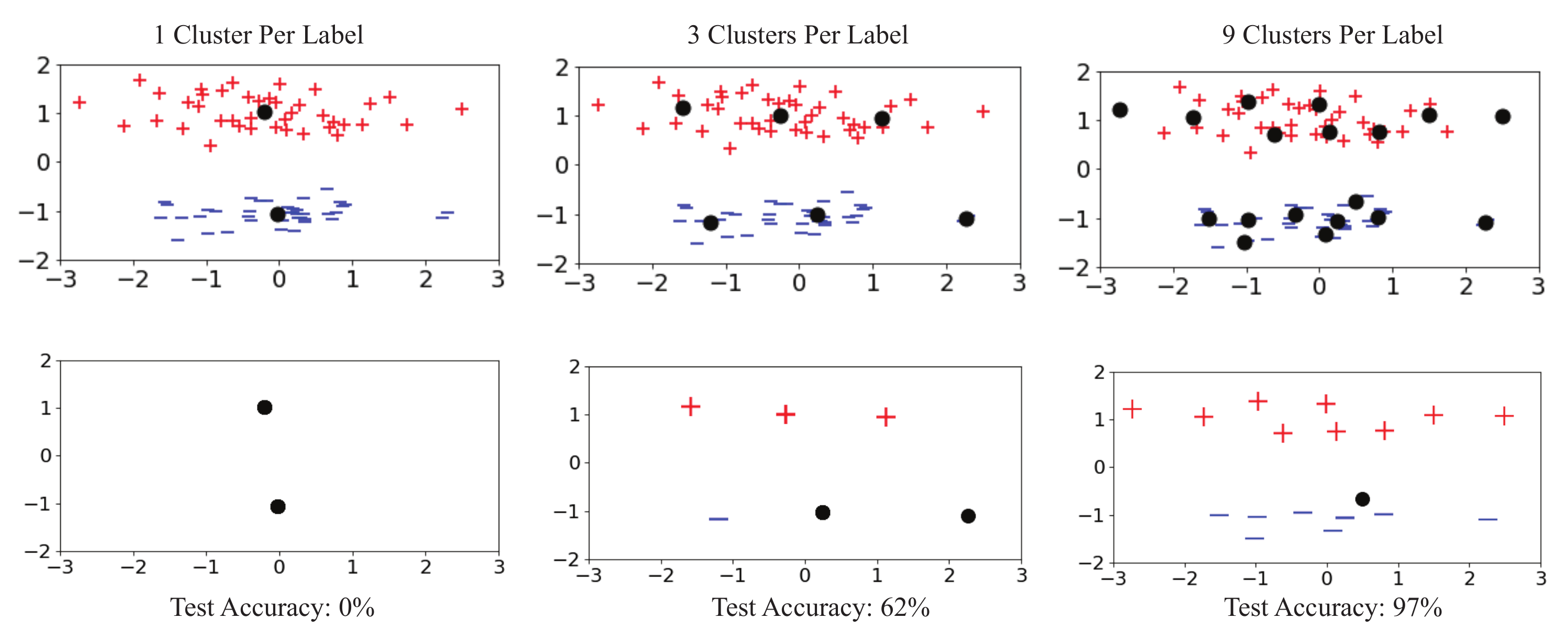}

\caption{Learning on a custom lattice determined through clustering, shown on synth02. Top row: Each label was clustered into (left-to-right) 1, 3, 9 clusters; cluster centers (black dots) were then used as atoms. Bottom row: Test prediction accuracy. correctly labeled positive points are denoted by $+$, correctly labeled negative points are denoted by $-$, mislabeled are denoted by black dots.}
\label{fig:table_quant}
\end{figure*}

Figure~\ref{fig:table_quant} presents the results of learning on the synth02 dataset under a fully custom quantization. This quantization was produced by clustering the positive and negative training examples separately using k-means clustering for k= 1, 3, and 9, and then taking the cluster centers and using those as the atoms. The top row displays the cluster centers in relation to the data. The bottom row shows the results of training Perceptron using only those lattice points, and then testing on the test set: a red plus denotes a correctly classified positively labelled test point, a blue minus denotes a correctly classified positively labeled test point, and the black dot denotes an incorrectly classified point. The coarsest quantization (left) contains little information -- the classification accuracy could either be 0 or 100; the two finer quantizations result in 62 and 97 percent accuracy, showing that it is possible to learn under a coarse custom quantization. Techniques for creating custom quantizations for a given dataset are left as future work.

\section{Related Work and Discussion}
\label{sec:related}

Studying the impact of numerical representations on learning was a topic of active interest in the context of neural networks in the nineties~\citep{holt1991back,hoehfeld1992learning,simard1994backpropagation} --- with focus on fixed point, floating point or even integer representations.  The general consensus of this largely empirical line of research suggested the feasibility of backpropagation-based neural network learning. Also related is the work on linear threshold functions with noisy updates~\citep{blum1998polynomial}.

In recent years, with the stunning successes of neural networks~\citep{goodfellow2016deep}, interest in studying numeric representations for learning has been re-invigorated~\citep[for example]{courbariaux2014training,gupta2015deep}.
In particular, there have been several lines of work focusing on convolutional neural networks~\citep[inter alia]{lin2016fixed,wu2016quantized,das2018mixed,micikevicius2017mixed} which show that tuning or customizing numeric precision does not degrade performance.

Despite the many empirical results pointing towards learnability with quantized representations, there has been very little in the form of theoretical guarantees. Only recently, we have started seeing some work in this direction~\citep{zhang2017zipml,alistarh2016qsgd,chatterjee2017towards}. The ZipML framework~\citep{zhang2017zipml} is conceptually related to the work presented here in that it seeks to formally study the impact of quantization on learning. But there are crucial differences in both formalization --- while this paper targets online updates (Perceptron and Frank-Wolfe), ZipML studies the convergence of stochastic gradient descent. Moreover, in this paper, we formally and empirically analyze quantized versions of {\em existing} algorithms, while ZipML proposes a new double-rounding scheme for learning.

Most work has focused on the standard fixed/floating point representations. However, some recent work has suggested the possibility of low bitwidth custom numeric representations tailored to learning~\citep{seide20141,hubara2016quantized,rastegari2016xnor,park2017weighted,zhang2017zipml,koster2017flexpoint}. Some of these methods have shown strong predictive performance with surprisingly coarse quantization (including using one or two bits per parameter!). The formalization for quantized learning presented in this paper could serve as a basis for analyzing such models.

Due to the potential power gains, perhaps unsurprisingly, the computer architecture community has shown keen interest in low bitwidth representations. For example, several machine learning specific architectures assume low precision representations~\citep{akopyan2015truenorth,shafiee2016isaac,jouppi2017datacenter,kara2017fpga} and this paper presents a formal grounding for that assumption. The focus of these lines of work have largely been speed and power consumption. However, since learning algorithms implemented in hardware {\em only} interact with quantized values to represent learned weights and features, guaranteed learning with coarse quantization is crucial for their usefulness. Indeed, by designing dataset- or task-specific quantization, we may be able to make further gains.

\section{Conclusion}

Statistical machine learning theory assumes we learn using real-valued vectors, however this is inconsistent with the discrete quantizations we are forced to learn with in practice. We propose a framework for reasoning about learning under quantization by abandoning the real-valued vector view of learning, and instead considering the subset of $\R^d$ which is represented precisely by a quantization. This framework gives us the flexibility to reason about fixed point, floating point, and custom numeric representation. We use this framework to prove convergence guarantees for quantization-aware versions of the Perceptron and Frank-Wolfe algorithms. Finally, we present empirical results which show that we can learn with much fewer than 64-bits, and which points we choose to represent is more important than how many points.

\subsubsection*{Acknowledgements}
Jeff Phillips thanks his support from NSF CCF-1350888, ACI-1443046,
CNS- 1514520, CNS-1564287, and IIS-1816149. Vivek Srikumar thanks
NSF EAGER-1643056 and a gift from Intel corporations.

\bibliography{cited}
\bibliographystyle{plainnat}

\end{document}